\documentclass[smallcondensed]{svjour3}                     

\smartqed  
\usepackage{graphicx}
\usepackage{algorithm}
\usepackage{algorithmic}

 \journalname{Machine Learning}
\begin{document}

\title{RandSVM: A Randomized Algorithm for training Support Vector Machines on Large Datasets
}

\titlerunning{Randomized Algorithms for Large scale SVMs}        

\author{Vinay Jethava  \and
        Krishnan Suresh \and
	Chiranjib Bhattacharyya\and 
	Ramesh Hariharan
}


\institute{Vinay Jethava \at
              Indian Institute of Science, Bangalore \\
              \email{vjethava@csa.iisc.ernet.in}           
           \and
           Krishnan Suresh \at
           Yahoo Labs, India \\
	   \email{krishnan.suresh@gmail.com} 
	\and
	Chiranjib Bhattacharyya \at
	Indian Institute of Science, Bangalore \\
	Tel.: +91-080-22532468 EXT: 240\\
        Fax: +91-080-23602911\\
	\email{chiru@csa.iisc.ernet.in}
	\and 
	Ramesh Hariharan \at
	Strand LifeSciences, Bangalore \\
	\email{ramesh@strandls.com}
}
\date{Received: date / Accepted: date}

\maketitle

\begin{abstract}
We propose a randomized algorithm for  
training Support vector machines(SVMs) on large datasets. 
By using ideas from Random projections we show that 
the combinatorial dimension of SVMs is $O(\mbox{log}~ n)$ with high probability.
This estimate of combinatorial dimension is used 
to derive an iterative algorithm, called RandSVM, 
which at each step calls an existing solver to train SVMs on 
a randomly chosen subset of size $O(\mbox{log}~ n)$.  
The algorithm has probabilistic guarantees and is 
capable of training SVMs 
with Kernels for both classification and regression problems.  Experiments done on synthetic and real life data sets demonstrate that the algorithm scales up 
existing SVM learners, without loss of accuracy.  

\end{abstract}

\section{Introduction}
\label{sec:intro}
Consider a training data set $D = \{(x_i , y_i ) , \; i = 1 \dots n\}$ where $x_i \in
R^d$ are data points and $y_i$ are labels. 
The problem of learning a linear classifier,
$y = sign (w^\top x + b)$,  where $y = \{1,-1\}$ or a linear function $y = w^\top x + b$ when $y$ is a scalar 
can be understood as estimating $\{w, b\}$ from $D$. 
Over the years Support Vector Machines(SVMs) have emerged as powerful tools
for estimating such functions. In this paper we concentrate on developing 
randomized algorithms for learning SVMs on large datasets. 
 For a 
detailed review of SVM classification and SVM regression please see \cite{Vapnik95}. 

To develop notation we briefly discuss the problem of training linear classifiers.
The SVM formulation for linearly separable datasets is given by \cite{Vapnik95} 
\begin{eqnarray}\label{eq:SVM}
& \min_{w,b} \frac{1}{2}||w||^2  \nonumber \\
& s.t. \, y_i (w^\top x_i + b) \ge 1 , i=1\dots n & \nonumber
\end{eqnarray} 
where $||w|| = \sqrt{w^\top w}$, is the euclidean norm of $w$.
The formulation has very interesting geometric underpinnings 
~\cite{Bennett00duality}. It can be understood as computing the distance between convex 
hulls of the sets $\{x_i |  y_i = 1\}$ and $\{x_j | y_j = -1 \}$. 
For linearly non-separable datasets the following formulation \\   
\textbf{C-SVM-1:} \\
\begin{eqnarray}\label{eq:C-SVM-1}
& \min_{(w,b,\xi)} \frac{1}{2}||w||^2 + C\sum_{i=1}^{n} \xi_i & \nonumber \\
& s.t. \, y_i (w^\top x_i + b) \ge 1 - \xi_i ,\quad \xi_i \ge 0, i=1\dots n & \nonumber
\end{eqnarray}
which will be called $C-SVM$, again due to \cite{Vapnik95}, can be used. This 
formulation do not have an elegant geometric interpretation like the separable 
case,  but one can consider C-SVMs as computing the distance between two 
reduced convex hulls ~\cite{Bennett00duality}.
  
Both the formulations are instances of \emph{Abstract Optimization Problem(AOP)}~\cite{Balcazar01a,Balcazar01b,Gartner92}. An AOP is defined as follows:
\begin{definition}[AOP]
\emph{ An AOP is a triple $(H, <, \Phi)$ where $H$ is a finite set, $<$ a total ordering on $2^H$, and $\Phi$ an oracle that, for a given $F \subseteq G \subseteq H$, either reports $F = min_< \{F' |F' \subseteq G\}$ or returns a set $F' \subseteq G$ with $F' < F$.}
\end{definition}
Every AOP has a \emph{combinatorial dimension associated} with it; the combinatorial dimension
captures the notion of number of free variables for that AOP. An AOP can be solved by a
randomized algorithm by selecting subsets of size greater than the combinatorial dimension
of the problem~\cite{Gartner92}. We wish to exploit this property 
of AOPs to design randomized algorithms for SVMs.  

The idea is to develop an iterative algorithm where in each step one needs 
to solve a SVM formulation on a small subset of the training data. Crucial 
to this idea is the size of the subset which is tied to the combinatorial 
dimension of the SVM formulation. To this end note that
at optimality $w$ is given by
\begin{equation}\label{defn:w}
w =\sum_{i:\alpha_i>0}\alpha_i y_i x_i,
\end{equation}
for both the separable and non-separable case.
Using the $\alpha$ variables one can define the set of Support vectors~(SVs),
\begin{equation}\label{defn:S}
S = \{x_i |\alpha_i > 0\}
\end{equation}
which defines $w$. The set $S$ may not be unique,
though $w$ is. 
The combinatorial dimension of SVMs is given by the minimum number of SVs  
required to define $w$. More formally
\begin{equation}\label{def:delta}
 \Delta = \min_{S} |S| 
\end{equation}
where $|S|$ is the cardinality of the set $S$.  

The parameter $\Delta$
does not change with number of examples $n$, and is often much less than $n$.
Apriori the value of $\Delta$ is not known, but for linearly separable classification problems the following holds: $2\le \Delta \le d+1$. This follows from the observation that 
it computes the distance between 2 non-overlapping convex hulls~\cite{Bennett00duality}. When the problem is not linearly separable, the reduced convex hull interpretation leads to a very crude upper bound, which is much larger than $d$.

The idea of iterating over randomly sampled subsets of size greater than 
$\Delta$, for training SVMs was first explored by 
~\cite{Balcazar01a,Balcazar01b}, and 
the resulting algorithm was called RandSVM. 
The  RandSVM procedure iterates over subsets of size proportional to $\Delta^2$ , as shown
in Algorithm~\ref{algo:RandSVM}.
However as the authors noted that RandSVM is not practical because of the following reasons. For linear classifiers the sample size is too large
in case of high dimensional data sets. 
For non-linear SVMs ~\cite{Vapnik95}  the dimension of feature space is usually unknown
when using kernels. Even in this case one can obtain a very crude 
upper-bound on $\Delta$ by the reduced convex hull approach but 
is not really useful as the
number obtained is very large.

\begin{algorithm}
\label{algo:RandSVM}
\caption{$RandSVM(D,\Delta)$}

\begin{algorithmic}[1]
\REQUIRE $D$ - Dataset
\REQUIRE $\Delta$ - Combinatorial Dimension

\STATE Sample size $r = 6\Delta^2$
\STATE Set weights $w(x_i)$ to be 1 for all examples in $D$. For any set $A \subseteq D$, let $w(A) =\sum_{x_i\in A} w(x_i)$.
\REPEAT
\STATE Select a sample $S$ of size $r$ randomly according to $w$.
\STATE Use a SVM solver to solve the smaller problem. Let the classifier obtained be $C$.
\STATE Classify the non sampled documents $D − S$.
\STATE Let $V$ be the set of misclassified documents and let $v$ be the size of $V$.
\IF{$(w(V ) \le w(D)/(3\Delta))$}
\STATE Double the weights of misclassified documents.
\ENDIF
\UNTIL $v=0$
\STATE Done  
\end{algorithmic}
\end{algorithm}

This work overcomes the above problems using ideas from random projections~\cite{JohnsonLindenstrauss84,dasgupta99elementary,Arriaga06rp} and randomized algorithms~\cite{Clarkson95,Gartner92,Gartner00}. As mentioned by the authors of RandSVM, the biggest bottleneck in their algorithm is the value of $\Delta$ as it is too large. The main contribution of this work is, using ideas from random projections, the conjecture that if RandSVM is solved using $\Delta$ equal to $O(\log n)$, then the
solution obtained is close to optimal with high probability(Theorem~\ref{thm:thm3}, particularly for linearly separable and almost separable data sets. Almost separable data sets are those which
become linearly separable when a small number of properly chosen data points are deleted from them. The second contribution is an algorithm which, using ideas from randomized
algorithms for Linear Programming(LP), solves the SVM problem by using samples of size
linear in $\Delta$. This work also shows that the theory can be applied to non-linear kernels. The formulation naturally applies to regression problems. 

The paper is organized as follows: Section~\ref{sec:past_work} introduces the previous work, Section~\ref{sec:classification} presents the improved algorithm for classification for almost linearly  separable data. 
 Section~\ref{sec:regression} presents the improved algorithm for the $\epsilon-$tube regression formulation. We present our results  and conclusions in Section~\ref{sec:results} and ~\ref{sec:conclusions}.

\section{Past Work}\label{sec:past_work}
We begin by reviewing some results from random projections~\cite{Arriaga06rp}. The data points in $R^d$ are projected into a random $k$ dimensional subspace where $k \ll d$. Then, we look at a few algorithms which focus on large scale classification.

\subsection{Random Projection}\label{subsec:random_projection}
The following lemma discusses how the $L_2$ norm of a vector is preserved when it is projected on a random subspace.
\begin{lemma}\label{thm:lemma_1}
 Let $R = (r_{i j} )$ be a random $d \times k$ matrix, such that each entry $(r_{i j})$ is chosen
independently according to $N(0, 1)$. For any fixed vector $u \in R^d$, and any $\epsilon> 0$, let
$u'=\frac{R^T u}{\sqrt{k}}$. Then $E[||u'||^2 ] = ||u||^2$ and the following bounds hold:
\begin{displaymath}
 (1-\epsilon) ||u||^2 \le ||u'||^2 \le (1+\epsilon) ||u||^2
\end{displaymath}
with probability at least $1-2e^{(\epsilon^2-\epsilon^3)\frac{k}{4}}$.
\end{lemma}
The following theorem and its corollary show the change in the Euclidean distance
between 2 points and the dot products when they are projected onto a lower dimensional
space~\cite{Arriaga06rp}.

\begin{lemma}\label{thm:lemma_2}
 Let $u, v \in R^d$. Let $u' = \frac{R^T u}{\sqrt k}$ and $v' = \frac{R^T v}{\sqrt k}$ be the projections of $u$ and $v$ to $R^k$ via a random matrix $R$ whose entries are chosen independently from $N (0, 1)$ or $U (−1, 1)$. Then for any $\epsilon> 0$, the following bounds hold
\begin{displaymath}
 (1-\epsilon) ||u-v||^2 \le ||u'-v'||^2
\end{displaymath}
with probability at least $1-e^{-(\epsilon^2-\epsilon^3)\frac{k}{4}}$ and
\begin{displaymath}
 ||u'-v'||^2 \le (1+\epsilon) ||u-v||^2
\end{displaymath}
with probability at least $1-e^{-(\epsilon^2-\epsilon^3)\frac{k}{4}}$.
\end{lemma}
A corollary of the above theorem shows how well the dot products are preserved upon projection(This is a slight modification of the corollary given in ~\cite{Arriaga06rp}).
\begin{corollary}\label{thm:corollary_1}
Let $u, v$ be vectors in $R^d$ s.t. $||u|| \le L_1 , ||v|| \le L_2$ . Let $R$ be a random matrix whose entries are chosen independently from either $N (0, 1)$ or $U (−1, 1)$. Define $u' = \frac{R^T u}{\sqrt k}$ and $v' = \frac{R^T v}{\sqrt k}$. Then for any $\epsilon> 0$, the following bound holds
\begin{displaymath}
 u\cdot v - \frac{\epsilon}{2} (L_1^2 + L_2^2) \le u' \cdot v' \le u\cdot v + \frac{\epsilon}{2} (L_1^2 + L_2^2)
\end{displaymath}
with probability at least $1-4e^{-\epsilon^2 \frac{k}{8}}$.
\end{corollary}
\begin{proof} For the vectors $u$ and $v$, let the event $E_1$ be
\begin{displaymath}
(1-\epsilon) ||u-v||^2 \le ||u'-v'||^2 \le (1+\epsilon) ||u-v||^2
\end{displaymath}
and $E_2$ be
\begin{displaymath}
(1-\epsilon) ||u+v||^2 \le ||u'+v'||^2 \le (1+\epsilon) ||u+v||^2
\end{displaymath}
Hence, from lemma~\ref{thm:lemma_2}:
\begin{displaymath}
 P(E_1 \textrm{ and } E_2) \ge 1 - 4e^{-(\epsilon^2-\epsilon^3)\frac{k}{4}}
\end{displaymath}
Now,
\begin{eqnarray*}
 u'\cdot v' & = & \frac{1}{4}(||u'+v'||^2 - ||u'-v'||^2) \\
 & \le & \frac{1}{4} \big((1+\epsilon)||u'+v'||^2 -(1-\epsilon) ||u'-v'||^2\big) \\
 & = & u\cdot v + \frac{\epsilon}{2} (||u||^2 +||v||^2) \\
 \Rightarrow u' \cdot v' & \le &  u\cdot v + \frac{\epsilon}{2} (L_1^2 +L_2^2) \\
\end{eqnarray*}
                                                                              
The above inequality holds with probability greater than or equal to $1 - 2e^{-(\epsilon^2-\epsilon^3)\frac{k}{4}}$. Similarly,
\begin{displaymath}
 u' \cdot v'  \ge   u\cdot v - \frac{\epsilon}{2} (L_1^2 +L_2^2)
\end{displaymath}
holds with probability greater than or equal to $1 - 2e^{-(\epsilon^2-\epsilon^3)\frac{k}{4}}$. \qed
\end{proof}

\subsection{Large scale classification}\label{subsec:past_SVM}
We look at a few algorithms which focus on large scale classification.~\cite{Fung01proximalsupport} presented a SVM formulation called Proximal SVM in which the objective is a non linear least squares function and the inequality constraints are replaced by a system of equations. Finding the best separating hyperplane now involves solving this system of equations. This is done by inverting a $d \times d$ matrix, as a result of which the method is not
feasible for datasets like text for which $d$ is very high. Also, the method involves a matrix
multiplication $H^T H$ where $H$ is a $n \times (d + 1)$ matrix. So the entire data matrix needs to
be kept in memory and hence the method is not scalable in terms of memory.

    \cite{Keerthi05} presented an algorithm L2-SVM-MFN which uses a conjugate gradient method to solve the SVM problem and thus does not have to perform any matrix inversion as the previous method. Results in their paper indicate that the algorithm performs very well for large high dimensional datasets like text. Analysis of the algorithm
indicates that it accesses the data vectors in a sequential manner and hence does not have to keep the data matrix in main memory, making it scalable in terms of memory.

    Our work is closely related to~\cite{Balcazar01a,Balcazar01b}. They propose that $d$ be used as
the combinatorial dimension of the problem for the separable case. The dual of the SVM
problem, when the data is linearly separable, is the minimum distance between the 2 convex
hulls of the positive and negative examples. When the data is not linearly separable, these
2 hulls overlap. This can reduced to the separable case, by condensing the 2 hulls~\cite{Bennett00duality}. This is done as follows. Let $Z$ be the set of composed examples $z_I$ where $z_I=\frac{x_{i_1} + \dots + x_{i_m}}{m}$, where each $x_{i_j}$ is a distinct element of $D$ and all the points defining a $z_I$ have the same label and the label of $z_I$ is the same(For details on this condensation, see their paper). In this case, we have $|Z| \le {n \choose m}$ and $(m + 1)\le \Delta \le  m(d + 1)$. It is this aspect of the SVM problem which was used by the authors to develop a randomized algorithm to solve the problem, given in Algorithm~\ref{algo:RandSVM}.

    The algorithm proceeds in multiple iterations, where in each iteration it picks up a
subset of the training data $S$, such that the size of the subset, $r$, is greater than the number
of support vectors. Any SVM solver can be used to train a classifier $C$ on the sampled
subset, which is smaller than the entire data. Based on the classifier $C$ obtained, the
sampling probabilities are changed for the training data such that in successive iterations,
the support vectors have a higher probability of selection. This process is repeated until the
number of misclassified documents $v = 0$. The termination of the algorithm is guaranteed
in a probabilistic fashion in~\cite{Clarkson95}. The authors recommend using $m (d + 1)$ as
an estimate of $\Delta$. This choice of $\Delta$ makes the subset size too large for high dimensional
datasets, making it impractical.

    To overcome this problem we use ideas from random projections~\cite{JohnsonLindenstrauss84,dasgupta99elementary,Arriaga06rp}. Consider projecting
the data points into a random $k$ dimensional subspace where $k << d$. Using this idea,
we give a theoretical bound on the combinatorial dimension $\Delta$ which is much lesser than
the original data dimension $d$, in the almost linearly separable case. In practice, it has
been observed that $\Delta$ is even lower. We then apply this to make the above algorithm scalable(without actually performing any random projection of the data).

\section{Classification}\label{sec:classification}
This section uses results from random projections, and randomized algorithms for linear
programming to develop a new algorithm for solving large scale SVM classification problems.
In Section~\ref{subsec:linsep_classificatn}, we discuss the case of linearly separable data and estimate a the number
of support vectors required such that the margin is preserved with high probability, and
show that this number is much smaller than the data dimension d, using ideas from random
projections. In Section~\ref{subsec:almostseparable}, we look at how the analysis applies to almost separable data and
present the main result of the paper(Theorem~\ref{thm:thm2}). The section ends with a discussion on the
application of the theory to non-linear kernels. In Section~\ref{subsec:randsvm}, we present the randomized
algorithm from SVM learning.

\subsection{Linearly separable data}\label{subsec:linsep_classificatn}
We start with determining the dimension $k$ of the target space such that on performing a
random projection to the space, the Euclidean distances and dot products are preserved.
The appendix contains a few results from random projections which will be used in this
section.
For a linearly separable data set $D = \{(x _i , y_i ),\; i = 1,\ldots, n\}, x_i \in R^d , y_i \in \{+1, -1\}$, the
C-SVM formulation is the same as $C-SVM-1$ with $\xi_i = 0,\, i = 1\ldots n$. By dividing all the
constraints with $||w||$, the problem can be reformulated as follows: \\
\textbf{C-SVM-2a:}\\
\begin{eqnarray}\label{eq:C-SVM-2a}
& \max_{(\hat{w},b,l)} l & \nonumber \\
& s.t. \, y_i (\hat{w} \cdot x_i + \hat{b}) \ge l , i=1\dots n, \quad ||\hat{w}||=1 & \nonumber
\end{eqnarray}
where $\hat{w}=\frac{w}{||w||}$, $\hat{b}=\frac{b}{||w||}$ and $\hat{l}=\frac{1}{||w||}$. $l$ is the margin induced by the separating hyperplanes, that is, it is the distance between the 2 supporting hyperplanes.

The determination of $k$ proceeds as follows. First, for any given value of $k$, we show
 the change in the margin as a function of $k$ when the data points are projected
onto the $k$ dimensional subspace and the problem solved. From this, we determine the value $k( k << d )$ which will
preserve margin with a very high probability. In a $k$ dimensional  subspace,
 there are at the most $k+1$ support vectors. Using the idea of \emph{orthogonal extensions}(definition
 appears later in this section), we prove that when the problem is solved in the original
space, using an estimate of $k+1$ on the number of support vectors, the margin is preserved with a very high 
probability.

Let $w'$ and $x_i',i=1,\dots ,n$ be the projection of $\hat{w}$ and $x_i,i=1,\dots ,n$  respectively 
onto a $k$ dimensional subspace (as in \textbf{Lemma~\ref{thm:lemma_2}}). The classification problem in the
projected space with the data set being $D'=\{(x_i',y_i),i=1,\dots,n\},x_i'\in R^k,y_i\in\{+1,-1\}$ can be written as follows:\\
\textbf{C-SVM-2b:}
\begin{displaymath}
Maximize_{(w',\hat{b},l')} l'
\end{displaymath}
\begin{displaymath}
Subject\mbox{ }to:\mbox{ } y_i(w'\cdot x_i' + \hat{b}) \geq l',\mbox{ }i=1\dots n,\mbox{ }||w'||  \leq  1
\end{displaymath}
where $l'=l(1-\gamma)$, $\gamma$ is the distortion and $0 < \gamma < 1$.
The following theorem predicts, for a given value of $\gamma$, the $k$ such that the margin
is preserved with a high probability upon projection.

\begin{theorem}\label{thm:thm1}
Let $L=max||x_i||$, and $(w^*,b^*,l^*)$ be the optimal solution for \textbf{C-SVM-2a}. Let $R$
be a random $d \times k$ matrix  as given in \textbf{Lemma~\ref{thm:lemma_2}}. Let $\widetilde{w}=\frac{R^T w^*}{\sqrt{k}}$ and $x_i'=\frac{R^T x_i}{\sqrt{k}}, i=1,\dots ,n$. If $k \geq \frac{8}{\gamma^2}(1+\frac{(1+L^2)}{2l^*})^2 \log \frac{4n}{\delta} ,\;0 < \gamma < 1,\mbox{ } 0 < \delta < 1$, then the following bound holds on the optimal margin $l_P$ obtained by solving the problem \textbf{C-SVM-2b}:
\begin{displaymath}
P( l_P \geq l^*(1-\gamma) ) \geq  1-\delta
\end{displaymath}
\end{theorem}
\begin{proof}
From Corollary~\ref{thm:corollary_1} of Lemma~\ref{thm:lemma_2}, we have
\begin{displaymath}
w^* \cdot x_i - \frac{\epsilon}{2}(1+L^2) \leq \widetilde{w}\cdot x_i' \leq w^* \cdot x_i + \frac{\epsilon}{2}(1+L^2) 
\end{displaymath}
which holds with probability at least $1-4e^{-\epsilon^2\frac{k}{8}}$, for some $\epsilon > 0$.
Consider some example $x_i$ with $y_i=1$. Then the following holds with probability at least $1-2e^{-\epsilon^2\frac{k}{8}}$
\begin{displaymath}
\widetilde{w}\cdot x_i' + b^* \geq w^* \cdot x_i - \frac{\epsilon}{2}(1+L^2) + b^* \geq l^* - \frac{\epsilon}{2}(1+L^2) 
\end{displaymath}
Dividing the above by $||\widetilde{w}||$, we have
\begin{displaymath}
\frac{\widetilde{w}\cdot x_i' + b^*}{||\widetilde{w}||} \geq \frac{l^* - \frac{\epsilon}{2}(1+L^2)}{||\widetilde{w}||}
\end{displaymath}
Note that from Lemma~\ref{thm:lemma_1}, we have $\sqrt{(1-\epsilon)}||w^*|| \leq ||\widetilde{w}|| \leq \sqrt{(1+\epsilon)||w^*||}$,
 with probability at least $1-2e^{-\epsilon^2\frac{k}{8}}$. Since $||w^*||=1$, we have
$\sqrt{1-\epsilon} \leq ||\widetilde{w}|| \leq \sqrt{1+\epsilon}$.\hspace{1cm}\linebreak
Hence
\begin{eqnarray*}
\frac{\widetilde{w}\cdot x_i' + b^*}{||\widetilde{w}||} & \geq & \frac{l^* - \frac{\epsilon}{2}(1+L^2)}{\sqrt{1+\epsilon}} \\
& \geq & (l^*-\frac{\epsilon}{2}(1+L^2))(\sqrt{1-\epsilon})\mbox{ } \geq \mbox{ }l^*(1-\frac{\epsilon}{2l^*}(1+L^2))(1-\epsilon)\\
& \geq & l^*(1-\epsilon - \frac{\epsilon}{2l^*}(1+L^2))\mbox{ }\geq\mbox{ }l^*(1-\epsilon(1 + \frac{1+L^2}{2l^*}))
\end{eqnarray*}
This holds with probability at least $1-4e^{-\epsilon^2\frac{k}{8}}$. 
A similar result can be derived for a point $x_j$ for which $y_j=-1$.
The above analysis guarantees that by projecting onto a $k$ dimensional space, there exists at least one hyperplane
$(\frac{\widetilde{w}}{||\widetilde{w}||},\frac{b^*}{||\widetilde{w}||})$, which guarantees a margin of $l^*(1-\gamma)$
where
\begin{equation}
\label{gamma-epsilon}
\gamma \leq \epsilon(1 + \frac{1+L^2}{2l^*})
\end{equation}
with probability at least $1-n4e^{-\epsilon^2\frac{k}{8}}$.
The margin obtained by solving the problem \textbf{C-SVM-2b}, $l_P$ can only be better than this.   
So the value of $k$ is given by:
\begin{equation}
n4e^{-\frac{ \gamma^2}{(1+\frac{1+L^2}{2l^*})^2}\frac{k}{8}}  \leq \delta\mbox{ }\Rightarrow\mbox{ }
k \geq \frac{8(1+\frac{(1+L^2)}{2l^*})^2}{ \gamma^2 } \log \frac{4n}{\delta} 
\end{equation} \qed
\end{proof}

So by randomly projecting the points onto a $k$ dimensional subspace, the margin
is preserved with a high probability. This result is similar to the results in large scale learning
using random projections~\cite{Arriaga06rp,BBV06ldm}.
But there are fundamental differences between the method proposed in this paper and the previous methods:
no random projection is actually done here, and no black box access to the data distribution is required.
We use Theorem~\ref{thm:thm1} to determine an estimate on the number of support vectors such that margin is
preserved with a high probability, when the problem is solved in the original space. This
is given in Theorem~\ref{thm:thm2} and is the main contribution of this section. The theorem is based
on the following fact: in a $k$ dimensional space, the number of support vectors is upper
bounded by $k + 1$. We show that this $k + 1$ can be used as an estimate of the number of
support vectors in the original space such that the solution obtained preserves the margin
with a high probability. We start with the following definition.
\begin{definition}[Orthogonal extension]\emph{
An orthogonal extension of a $(k-1)$-dimensional flat( a $(k-1)$ dimensional flat
is a $(k-1)$-dimensional affine space) $h_p=(w_p,b)$, where $w_p=(w_1,\dots,w_k)$,
in a subspace $S_k$ of dimension k to a $d-1$-dimensional hyperplane $h=(\widetilde{w},b)$
in $d$-dimensional space, is defined as follows. Let $R \in R^{d\times d}$ be 
a random projection matrix as in \textbf{Lemma 2}. Let $\hat{R} \in R^{d \times k}$ be a another random
projection matrix which consists of only the the first $k$ columns of $R$. Let $\hat{x_i} = R^T x_i$ and 
$x_i' = \frac{\hat{R}^T}{\sqrt{k}}x_i$.Let $w_p=(w_1,\dots,w_k)$ be the optimal hyperplane classifier with margin $l_P$ for the points $x_1',\dots,x_n'$ in the $k$ dimensional subspace. Now define $\widetilde{w}$ to be all 0's in the last $d-k$ coordinates and identical to $w_p$ in the first $k$ coordinates, that is, $\widetilde{w}=(w_1,\dots,w_k,0,\dots,0)$. Orthogonal extensions have the following key property.
If $(w_p,b)$ is a separator with margin $l_p$ for the projected points, then  its orthogonal extension $(\widetilde{w},b)$ is a separator with margin $l_p$ for the original points,that is, if $y_i(w_p\cdot x_i' + b) \geq l,\; i=1,\dots,n$, then $y_i(\widetilde{w}\cdot \hat{x_i} + b) \geq l,\;i=1,\dots,n$.
}
\end{definition}
An important point to note, which will be required when extending orthogonal extensions to non-linear
kernels, is that dot products between the points are preserved upon doing orthogonal projections, that
is, $x_i'^TX_j' = \hat{x_i}^T\hat{x_j}$.

Let $L, l^*, \gamma, \delta \mbox{ and } n$ be as defined in Theorem~\ref{thm:thm1}. The following
is the main result of this section.
\begin{theorem}\label{thm:thm2}
Given $ k  \geq  \frac{8}{\gamma^2}(1+\frac{(1+L^2)}{2l^*})^2\log \frac{4n}{\delta}$ and $n$ training points
with maximum norm $L$ in $d$ dimensional space and separable 
by a hyperplane with margin $l^*$,
there exists a subset of $k'$ training points $x_1\ldots x_k'$ where $k' \leq k$ and
a hyperplane $h$ satisfying the following conditions:
\begin{enumerate}
\item $h$ has margin at least $l^*(1-\gamma)$ with probability at least $1-\delta$
\item $x_1\ldots x_k'$ are the only training points which lie either on $h_1$ or on $h_2$
\end{enumerate}
\end{theorem}
\begin{proof}
 Let $w^*,b^*$ denote the normal to a separating hyperplane with margin $l^*$,
 that is, $y_i(w^*\cdot x_i+b^*) \geq l^*$ for all $x_i$ and $||w^*||=1$.
 Consider a random projection of $x_1,\ldots ,x_n$  to a $k$ dimensional
 space and let $w',z_1,\ldots ,z_n$ be the projections of $w^*,x_1,\ldots ,x_n$,
 respectively, scaled by $1/\sqrt{k}$. 
 By Theorem 1,  
 $y_i(w' \cdot z_i+b^*/||w'||) \geq l^*(1-\gamma)$ holds for all $z_i$ with 
 probability at least $1-\delta$. Let $h$ be the {orthogonal extension}
  of $(w',b^*/||w'|)$ to the full $d$ dimensional space. 
 Then $h$ has margin at least $l^*(1-\gamma)$, as required.
 This shows the first part of the claim.

 To prove the second part, consider the projected 
 training points which lie on 
 either of the two supporting hyperplanes. Barring degeneracies,
 there are at the most $k$ such points. 
 Clearly, these will be the only points which lie on the 
 orthogonal extension $h$, by definition.\qed
 \end{proof}

From the above analysis, it is seen that if $k << d$, then we can estimate that the
number of support vectors is $k+1$, and the algorithm RandSVM would take on 
average $O(k\log n)$ iterations to solve the problem \cite{Balcazar01a,Balcazar01b}.

\subsection{Almost separable data}
\label{subsec:almostseparable}
In this section, we look at how the above analysis can be applied to \emph{almost separable}
data sets. We call a data set \emph{almost separable} if by removing a fraction $\kappa=O(\frac{\log n}{n})$
of the points, the data set becomes linearly separable. 

The C-SVM formulation when the data is not linearly separable(and \emph{almost separable})
was given in \textbf{C-SVM-1}. This problem can be reformulated as follows:
\begin{displaymath}
Minimize_{(w,b,\xi)} \displaystyle\sum_{i=1}^{n}\xi_{i}
\end{displaymath}
\begin{displaymath}
Subject\mbox{ }to:\mbox{ } y_i(w \cdot x_i + b) \geq l - \xi_i,\mbox{ }\xi_i \geq 0,\mbox{ }i=1\dots n;
||w|| \leq \frac{1}{l}
\end{displaymath}
This formulation is known as the \emph{Generalized Optimal Hyperplane} formulation.
Here $l$ depends on the value of $C$ in the C-formulation. At optimality, the
 margin $l^* = l$.
 The following theorem proves a result for almost separable data similar to the one proved in Theorem~\ref{thm:thm2}
for separable data.

\begin{theorem}\label{thm:thm3}
Given $k \geq \frac{8}{\gamma^2}(1+\frac{(1+L^2)}{2l^*})^2\log \frac{4n}{\delta} \mbox{ } + \kappa n$, 
$l^*$ being the margin at optimality, $l$ the lower bound on $l^*$ as in the Generalized Optimal 
Hyperplane formulation and $\kappa=O(\frac{\log n}{n})$, 
there exists a subset of $k'$ training points $x_1\ldots x_k$, $k' \leq k$ and a hyperplane $h$ satisfying 
the following conditions:
\begin{enumerate}
\item $h$ has margin at least $l(1-\gamma)$ with probability at least $1-\delta$
\item At the most $\frac{8(1+\frac{(1+L^2)}{2l^*})^2}{ \gamma^2 } \log \frac{4n}{\delta}$ points lie on the planes $h_1$ or on $h_2$
\item $x_1,\dots,x_k'$ are the only points which define the hyperplane $h$, that is, they are the support vectors of $h$.
\end{enumerate}
\end{theorem}
\begin{proof}
Let the optimal solution for the generalized optimal hyperplane formulation be $(w^*,b^*,\xi^*)$.  
$w^*=\displaystyle \sum_{i:\alpha_i>0}\alpha_i y_i x_i$, and $l^*=\frac{1}{||w^*||}$ as mentioned before. 
The set of support vectors can be split into to 2 disjoint sets,$SV_1  =  \{x_i:\alpha_i > 0 \mbox{ and } \xi^*_i = 0\}$(unbounded SVs)
 and $SV_2 = \{x_i:\alpha_i > 0 \mbox{ and } \xi^*_i > 0\}$(bounded SVs).

Now, consider removing the points in $SV_2$ from the data set. Then the data set
becomes linearly separable with margin $l^*$. Using an analysis similar to Theorem~\ref{thm:thm1},
and the fact that $l^* \geq l$, we have the proof for the first 2 conditions. 

When all the points in $SV_2$ are added back, at most all these points are added to the set
of support vectors and the margin does not change; 
this is guaranteed by the fact that we have assumed the worst possible margin for proving conditions 1 and 2,
and any value lower than this would violate the constraints of the problem. This proves condition 3. 
\qed 
\end{proof}
Hence the number of support vectors, such that the margin is preserved with high probability, is
\begin{equation}
\label{kVal}
  k + 1 = \frac{8}{\gamma^2}(1+\frac{(1+L^2)}{2l^*})^2\log \frac{4n}{\delta} \mbox{ } + \kappa n + 1 
        =  \frac{8}{\gamma^2}(1+\frac{(1+L^2)}{2l^*})^2\log \frac{4n}{\delta} \mbox{ } + O(\log n) 
\end{equation}
\paragraph{Using a non-linear kernel:}
Consider a mapping function $\Phi:R^d\rightarrow R^{d'},\mbox{ }d'>d$, which maps a point $x_i\in R^d$
to a point $z_i\in R^{d'}$, where $R^{d'}$ is a Euclidean space. Let the points $z_1,\dots,z_n$ be projected onto a random
$k$ dimensional subspace as before. The lemmas in the appendix
are applicable to these random projections\cite{BBV06ldm}. The orthogonal extensions can be considered as an projection
from the $k$ dimensional space to the $\Phi$-space, such that the kernel function values are preserved.
Then it can be shown that Theorem~\ref{thm:thm3} applies when using non-linear kernels also.

\subsection{A Randomized Algorithm}
\label{subsec:randsvm}
The reduction in the sample size from $6d^2$ to $6k^2$ is not enough to make RandSVM useful in practice 
as $6k^2$ is still a large number. This section presents another randomized algorithm which only requires that 
the sample size be greater than the number of support vectors. Hence a sample
size linear in $k$ can be used in the algorithm.  
This algorithm was first proposed to solve large scale LP problems~\cite{Pell01}; it
has been adapted for solving large scale SVM problems. The

\begin{algorithm}
\caption{RandSVM-1(D,k,r)}
\label{algo:LinearRandSVM}
\begin{algorithmic}[1]
\REQUIRE $D$ - The data set.
\REQUIRE $k$ - The estimate of the number of support vectors.
\REQUIRE $r$ - Sample size = $c k, c > 0$.
\STATE $S$ = randomsubset($D, r$); // \emph{Pick a random subset, $S$, of size $r$ from the data set $D$}
\STATE $SV$ = svmlearn($\{\}, S$); // \emph{SV - set of support vectors obtained by solving the problem $S$}
\STATE $V = \{x \in D-S | violates(x,SV)\}$ //\emph{violator - nonsampled point not satisfying KKT conditions}
\WHILE{$(|V| > 0$ and $|SV| < k)$}
\STATE $R$ = randomsubset($V$, $r - |SV|$); //\emph{Pick a random subset from the set of violators}
\STATE $SV'$ = svmlearn($SV,R$); //\emph{SV' - set of support vectors obtained by solving the problem $SV \cup R$}
\STATE $SV=SV'$;
\STATE $V = \{x \in D-(SV\cup R) | violates(x,SV)\}$; //\emph{Determine violators from nonsampled set}
\ENDWHILE
\STATE return $SV$
\end{algorithmic}
\end{algorithm}
\paragraph{Proof of Convergence:} Let $SV$ be the current set of support vectors. Condition $|SV|<k$ comes from Theorem 3. Hence if the condition is violated, then the algorithm terminates with a solution which is near optimal with a very high probability.\newline
 Now consider the case where $|SV|<k$ and $|V| > 0$. Let $x_i$ be a violator($x_i$ is a non-sampled point such that
$y_i(w^T x_i + b) < 1$). Solving the problem
with the set of constraints as $SV \cup {x_i}$ will only result, since SVM is an instance of AOP, in the increase(decrease)
of the objective function of the primal(dual). As there are only finite number of basis for an AOP, the algorithm
is bound to terminate; also if termination happens with the number of violators equal to zero, then the solution obtained is optimal.
\paragraph{Determination of k:} The value of $k$ depends on the margin $l^*$ which is not available in case of $C$-SVM. This can be
handled only by solving for $k$ as a function of $\epsilon$, where $\epsilon$
is as defined in the appendix and Theorem 1.
This can be done by combining Equation~\ref{gamma-epsilon} with Equation~\ref{kVal}:
\begin{equation}
k  \geq  \frac{8}{\gamma^2}(1+\frac{(1+L^2)}{2l^*})^2\log \frac{4n}{\delta} + O(\log n)
   \geq  \frac{16}{\gamma^2}(1+\frac{(1+L^2)}{2l^*})^2\log \frac{4n}{\delta} )
   \geq  \frac{16}{\epsilon^2}\log{\frac{4n}{\delta}}
\end{equation}

\section{Regression}\label{sec:regression}
Let us define a dataset $D = \{(x_i,y_i)| 1 \le i \le n, x_i \in R^d, y_i \in R\}$ to be linear, for a fixed $\epsilon \ge 0$, if the following 
formulation is feasible.\\
\textbf{SVR-1:} \\
\begin{eqnarray*}\label{eq:SVR-1}
& \min_{w,b} \frac{1}{2} ||w||^2 & \\
& \textrm{subject to: } \; y_i - w \cdot x_i - b \le \epsilon & \\
& \quad \quad \quad\quad \quad w \cdot x_i + b - y_i \le \epsilon &
\end{eqnarray*}
 This is the SVM regression formulation in which $D$ is constrained to lie
in a $\epsilon-$tube. The lagrangian is given as $\mathcal{L}(w, b, \alpha_i^+, \alpha_i^- )= \frac{1}{2} w^T w + \sum_i \alpha_i^+(y_i - w \cdot x_i -b -\epsilon) +  \sum_i \alpha_i^-( w \cdot x_i  + b - y_i -\epsilon)$. By KKT condition, the optimal solution will have $w^* = \sum_i( {\alpha_i^+}^* - {\alpha_i^-}^*) x_i$. The set of support vectors is union of two disjoint sets given as: $\{i: {\alpha_i^+}^* > 0\}\cup\{i: {\alpha_i^-}^* > 0\}$. We would like to develop randomized algorithms which
can solve such problems where $d$ and $n$ are large.
 
Let $x'_i$ be  the projection of $x_i, i = 1 \ldots n$ onto a
$k-$dimensional subspace. The regression problem in the projected space 
is given by\\
\textbf{SVR-2:} \\
\begin{eqnarray*}\label{eq:SVR-2}
 & \min_{w',b} \frac{1}{2} ||w'||^2 & \\
 & \textrm{subject to:}  \; y_i - w' \cdot x'_i - b \le \epsilon'& \\
& \quad \quad \quad\quad \quad w' \cdot x'_i + b - y_i \le \epsilon' &
\end{eqnarray*}
where $\epsilon'= \epsilon(1 + \gamma)$; $\gamma$ is the distortion. The following theorem predicts the value of $k$  
such that the $\epsilon$-tube is preserved, with a minor distortion, with a high probability upon projection.

\begin{theorem}\label{thm:regression}
 Let $L = \max ||x_i ||$, and $(w^* , b^* ),\;||w^∗ || = W$ be the optimal solution for $SVR-1$. Let $R$ be a random $d\times k$ matrix as given in Lemma~\ref{thm:lemma_2}. Let $\tilde{w}=\frac{R^T w^*}{\sqrt k}$ and $x_i'=\frac{R^T x_i}{\sqrt k},\;i=1,\ldots,n$. If $k\ge \frac{32(W^2+L^2)}{\gamma^2}\log{\frac{4n}{\delta}},\;0<\delta<1$, then the following bound holds on the optimal regressor $(w_P,b_P)$ obtained by solving the problem $\mathbf{SVR-2}$:
 \begin{displaymath}
  P(|w_P \cdot x'_i+b_P - y_i|\le \epsilon(1+\gamma) )\ge 1-\delta
 \end{displaymath}
\end{theorem}
\begin{proof}
From Corollary~\ref{thm:corollary_1} of Lemma~\ref{thm:lemma_2}, we have:
\begin{displaymath}
 w^*\cdot x_i -\frac{\epsilon_1}{2}(W^2+L^2) \le \tilde{w}\cdot x'_i\le w^*\cdot x_i + \frac{\epsilon_1}{2}(W^2+L^2)
\end{displaymath}
which holds with probability at least $1-4e^{-\epsilon_1^2\frac{k}{8}}$. So,
\begin{eqnarray*}
 \tilde{w} \cdot x'_i + b^* -y_i & \le & w^*\cdot x_i + b^* -y_i + \frac{\epsilon_1}{2}(W^2+L^2) \\
 & \le & \epsilon + \frac{\epsilon_1}{2}(W^2+L^2) = \epsilon \big(1+ \frac{\epsilon_1}{2 \epsilon}(W^2+L^2)\big)
\end{eqnarray*}
holds with probability at least $1-2e^{-\epsilon_1^2\frac{k}{8}}$. Similarly 
\begin{displaymath}
 y_i - \tilde{w} \cdot x'_i  -  b^*   \le  \epsilon + \frac{\epsilon_1}{2}(W^2+L^2)
\end{displaymath}
holds with probability at least $1-2e^{-\epsilon_1^2\frac{k}{8}}$. The above analysis guarantees that upon projection
onto a $k-$dimensional plane, there exists $(\tilde{w},b^*)$  which guarantees an $\epsilon-$tube of $\epsilon(1+\gamma)$, where
\begin{displaymath}
 \gamma \le  \frac{\epsilon_1}{2 \epsilon}(W^2+L^2)
\end{displaymath}
with probability at least $1-4e^{-\epsilon_1^2\frac{k}{8}}$. So the value of k is given by:
\begin{equation}
 n e^{-\Big(\frac{2 \gamma \epsilon}{W^2+L^2}\Big)^2 \frac{k}{8} } \le \delta \Rightarrow k \ge 2 \Big( \frac{W^2+L^2}{\gamma \epsilon }\Big)^2 \log \frac{4 n}{\delta}
\end{equation}
So, upon projection, there exists a regressor which preserves the $\epsilon$-tube with a high probability. The regressor obtained by solving $SVR-2$ can only do better than this.
\end{proof}
Let $L,W,\delta, \gamma, n, \epsilon$ be as defined in Theorem~\ref{thm:regression}, and $(\check{w}, b)$ be the orthogonal extension of $(w', b)$ to $R^d$ as in Lemma~\ref{thm:lemma_2}. Then, we get:
\begin{theorem}\label{thm:regression2}
 Given $ k  \geq 2 \Big( \frac{W^2+L^2}{\gamma \epsilon }\Big)^2 \log \frac{4 n}{\delta}$  and $n$ training points
with maximum norm $L$ in $d$ dimensional space for which the \textbf{SVR-1} problem with margin $\epsilon$ has a solution, there exists a subset of $k'$ training points $x_1\ldots x_{k'}$ where $k' \leq k$ and $h = \{(x,y)| y - w\cdot x -  b =\epsilon\}\bigcup \{(x,y)| w\cdot x + b - y = \epsilon\}$  satisfying the following conditions:
\begin{enumerate}
\item $(w,b)$ is the solution to a \textbf{SVR-1} with margin
at most $\epsilon(1+\gamma)$.
\item $x_1\ldots x_{k'}$ are the only training points which are in $h$.
\end{enumerate}
\end{theorem}
\begin{proof}
 Let $w^*,b^*$ denote the optimal regressor for problem \textbf{SVR-1} with margin $\epsilon$,
 that is, $w^*\cdot x_i+b^* - y_i \le \epsilon$ and $y_i - w^*\cdot x_i- b^*  \le \epsilon$ for all $x_i$. Let $w$ and $x'_i$ be the random projection of $w^*$ and $x_i$ as outlined in Theorem~\ref{thm:regression}. Then, $ | w' \cdot x_i' +b^* -y_i| \le \epsilon (1+\gamma)$ with probability at least $(1-\delta)$.  Let $(w, b^*)$ be the {orthogonal extension} of $(w',b^*)$ to the full $d$ dimensional space.
\begin{displaymath}
|w \cdot x_i + b^* - y_i| = | w' \cdot x_i' + b^* -y_i| \le \epsilon(1+\gamma)
\end{displaymath}
Therefore, $(w, b^*)$ is a solution to \textbf{SVR-1} with margin at most $\epsilon(1+\gamma)$.

To prove the second part, consider the projected training points which lie on $h' = \{(x',y)| y - w'\cdot x' -  b^* =\epsilon(1+\gamma)\}\bigcup \{(x',y)| w'\cdot x' + b^* - y = \epsilon(1+\gamma)\}$. Barring degeneracies,  there are at the most $k$ such points.  Clearly, these will be the only points which lie on the  orthogonal extension $h$, by  definition.
\end{proof}

Consider the problem: \textbf{SVR-3:}
\begin{eqnarray*}\label{eq:SVR-3}
& \min_{w, \xi_i} \frac{1}{2} ||w||^2 + \sum \xi_i & \\
& \textrm{subject to: } \; y_i - w \cdot x_i - b \le \epsilon +\xi_i & \\
& \quad \quad \quad\quad \quad w \cdot x_i + b - y_i \le \epsilon + \xi_i, & \xi_i \ge 0
\end{eqnarray*}

Analogous to the notion of \emph{almost separability} in the context of 
classification we define the notion of \emph{almost linear} as follows: 
the data set $D=\{(x_i,y_i)\}_{i=1}^n$ is almost linear if by removing a fraction $\kappa = O(\frac{\log n}{n})$ of the points, there exists a solution to the $SVR-1$ problem for some chosen $\epsilon > 0$. The problem $SVR-3$ is almost linear, if the optimal solution $(w^*, b^*,\xi^*)$ has the cardinality of the set $\{i: \xi^*_i > 0 \}$ as $O(\log{n} / n)$. This next theorem presents the result for almost separable data set for regression.

\begin{theorem}\label{thm:regression3}
Given $ k  \geq 2 \Big( \frac{W^2+L^2}{\gamma \epsilon }\Big)^2 \log \frac{4 n}{\delta} +\kappa n$ and $n$ training points
with maximum norm $L$ in $d$ dimensional space for which the $SVR-3$ problem with margin $\epsilon$ has an almost separable optimal solution, there exists a subset of $k'$ training points $x_1\ldots x_k'$ where $k' \leq k$ and 
 $h = \{(x,y)| y - w\cdot x -b = \epsilon\}\bigcup\{(x,y)|w\cdot x+b -y = \epsilon \} $ satisfying the following conditions:
\begin{enumerate}
\item $(w,b,\xi)$ is the solution to a hard $\epsilon-$tube regression problem with margin $\epsilon(1+\gamma)$.
\item At the most $2 \Big( \frac{W^2+L^2}{\gamma \epsilon }\Big)^2 \log \frac{4 n}{\delta}$ points lie on the plane $h$.
\item $x_1\ldots x_k'$ are the only training points which lie on $h$.
\end{enumerate}
\end{theorem}
\begin{proof}
Let the optimal solution for the \textbf{SVR-3} formulation be $(w^*,b^*,\xi^*)$. The set of support vectors can be split into to 2 disjoint sets,$SV_1  =  \{x_i:\alpha_i > 0 \mbox{ and } \xi^*_i = 0\}$(unbounded SVs)
 and $SV_2 = \{x_i:\alpha_i > 0 \mbox{ and } \xi^*_i > 0\}$(bounded SVs).

Now, consider removing the points in $SV_2$ from the data set. Then the data set
becomes linearly separable. Using an analysis similar to Theorem~\ref{thm:regression}, we have the proof for the first 2 conditions.

When all the points in $SV_2$ are added back, at most all these points are added to the set
of support vectors and the margin $\epsilon(1+\gamma)$ does not change; 
this is guaranteed by the fact that we have assumed the worst possible margin for proving conditions 1 and 2,
and any value lower than this would violate the constraints of the problem. This proves condition 3. 
\end{proof}

\section{Experiments}\label{sec:results}
\subsection{Classification}\label{subsec:results_classification}
This section discusses the performance of RandSVM in practice. The experiments were
performed on 4 data sets: 3 synthetic and 1 real world. RandSVM was used with LibSVM
as the solver when using a non-linear kernel; with SVMLight for a linear kernel. RandSVM
has been compared with state of the art SVM solvers: LibSVM~\cite{CC01a} for non-linear kernels, and
SVMPerf\footnote{http://svmlight.joachims.org/} 
nd SVMLin\footnote{http://people.cs.uchicago.edu/~vikass/svmlin.html} 
for linear kernels.

\subsubsection{Synthetic data sets}
The twonorm data set is a 2 class problem where each class is drawn from a multivariate normal distribution with unit variance. Each vector is a 20 dimensional vector. One class has mean $(a, a, \ldots, a)$, and the other class has mean $(-a, -a,\ldots, -a)$, where $a = \frac{2}{\sqrt{20}}$. The ringnorm data set is a 2 class problem with each vector consisting of 20 dimensions. Each class is drawn from a multivariate normal distribution. One class has mean 1, and
covariance 4 times the identity. The other class has mean $(a, a,\ldots, a)$, and unit covariance where $a = \frac{2}{\sqrt {20}}$.

The checkerboard data set consists of vectors in a 2 dimensional space. The points are
generated in a $4\times4$ grid. Both the classes are generated from a multivariate uniform distribution; each point is $(x1 = U (0, 4), x2 = U (0, 4))$. The points are labeled as follows - if $(x1 \%2 = x2 \% 2)$, then the point is labeled negative, else the point is labeled positive.For each of the synthetic data sets, a training set of 10,00,000 points and a test set of 10,000 points was generated. A smaller subset of 1,00,000 points was chosen from training 
set for parameter tuning. From now on, the smaller training set will have a subscript of 1 and the larger training set will have a subscript of 2, for example, ringnorm 1 and ringnorm2 .

\subsubsection{Real world data set}
The RCV1~\cite{RCV1} data set consists of 804,414 documents, with each document consisting of 47,236
features. Experiments were performed using 2 categories of the data set - CCAT and C11.
The data set was split into a training set of 7,00,000 documents and a test set of 104,414
documents.

     Table~\ref{tab:classification} shows the kernels which were used for each of the data sets. The parameters
used ($\sigma$ and $C$ for Gaussian kernels, and $C$ for linear kernels) were obtained by tuning using
grid search.

\paragraph{Selection of $k$ for RandSVM:} The values of $\epsilon$ and $\delta$ were fixed to 0.2 and 0.9 respectively,
for all the data sets. For linearly separable data sets, $k$ was set to $(16 \log(4n/\delta))/\epsilon^2$ . For
the others, $k$ was set to $(32 \log(4n/\delta))/ \epsilon^2$.

\subsubsection{Discussion of results:}
Table~\ref{tab:classification} has the timing and classification accuracy comparisons. The subscripts 1 and 2 indicate that the corresponding training set  sizes are $10^5$ and $10^6$ respectively. A '-' indicates that the solver did not finish  execution even after a running for a day. A 'X' indicates that the experiment is  not applicable for the corresponding solver. The ’∗’ indicates that the solver used with RandSVM was SVMLight; otherwise it was LibSVM.

The table shows that RandSVM can scale up SVM solvers for very large data sets. Using just a small wrapper around the solvers, RandSVM has scaled up SVMLight so that its performance is comparable to that of state of the art solvers such as SVMPerf and SVMLin. Similarly LibSVM has been made capable of quickly solving problems which it could not do before, even after executing for a day. In the case of ringnorm 1 dataset, the time taken by LibSVM is very small. Hence not much advantage is gained by solving smaller
sub-problems; this combined with the overheads involved in RandSVM resulted in such a slow execution. Hence RandSVM may not always be suited in the case of small datasets.

It is clear, from the experiments on the synthetic data sets, that the execution times taken by RandSVM for training with $10^5$ examples and $10^6$ examples are not too far apart; this is a clear indication that the algorithm scales well with the increase in the training set size.

All the runs of RandSVM except ringnorm 1 terminated with the condition $|SV| < k$ being violated. Since the classification accuracies obtained by using RandSVM and the baseline solvers are very close, it is clear that Theorem~\ref{thm:thm3} holds in practice.
\begin{table}[t]
\caption{Classification: Timing and accuracy(in brackets) comparison}
\centering
\label{tab:classification}       
\begin{tabular}{llllll}
\hline\noalign{\smallskip}
Category & Kernel & RandSVM & LibSVM & SVMPerf & SVMLin \\
\hline 
\smallskip
  twonorm1   & Gaussian & 300 (94.98\%) & 8542 (96.48\%) &        X  &           X \\
  twonorm2   & Gaussian & 437 (94.71\%) &        -      &        X   &          X \\
  ringnorm1  &  Gaussian & 2637 (70.66\%) &  256 (70.31\%)   &      X       &      X \\
  ringnorm2   & Gaussian & 4982 (65.74\%) & 85124 (65.34\%)  &      X       &      X \\
checkerboard1 & Gaussian & 406 (93.70\%) & 1568.93 (96.90\%) &      X       &      X \\
checkerboard2 & Gaussian & 814 (94.10\%) &        -       &       X        &     X \\
    CCAT∗    &  Linear  & 345 (94.37\%) &        X       & 148 (94.38\%)  &  429(95.1913\%) \\
     C11∗    &  Linear  & 449 (96.57\%) &        X       & 120 (97.53\%)  &  295 (97.71\%) \\

\noalign{\smallskip}\hline
\end{tabular}
\end{table}

%

\subsection{Regression}\label{subsec:results_regression}
The experiments were done on 1 synthetic datasets and 2 real world datasets - Forest Cover~\cite{Blackard00covtype} and MNIST\footnote{http://yann.lecun.com/exdb/mnist/}. RandSVM was compared with SVMLight~\cite{Joachims99a} and LibSVM~\cite{CC01a}. Table~\ref{tab:regression} gives the execution time(in seconds), mean square error(MSE) and correlation coefficient($\rho$) for $\epsilon-$regression. A linear kernel is used unless specified. The value of $k$ is calculated according to $k=\frac{32 \log(4 n /\delta)}{(\epsilon')^2}$. A value of $\epsilon'=0.2$ and $\delta=0.1$ is used. The datasets are as following: 

\subsubsection{Synthetic:} The input attributes $(x_1,\ldots,x_{10})$ are generated independently, each of which is distributed uniformly over $[0,1]$. The target is defined by $y = 10\sin(\pi x_1 x_2) + 20(x_3 - 0.5)  + 10 x_4 +
5 x_5 + N(0, 1)$. A value of $\epsilon=1.0$ is chosen. Two run are done for training set size of (a) $10^4$ and (b) $10^5$ respectively. 

\subsubsection{Forest Cover:}  There are 581012 records with label in $\{0,\ldots,6\}$ and $54$ features. The classification problem was transformed into a regression problem as follows: 
\begin{itemize}
 \item[a)] Predict the class labels with features scaled to   $[0,1]$ and $\epsilon=0.1$.
 \item[b)] Predict +1 for examples for class 2 and -1 for examples of other classes. Since class 2 is over represented, this leads to a more balanced problem. The features are scaled to $[0,1]$ and a value of $\epsilon=0.1$ is chosen. 
\end{itemize}

\subsubsection{MNIST:} The data has 60000 training points and 10000 test points. There are 784 features each in $\{0,\ldots,255\}$ and 10 class labels $\{0,\ldots,9\}$ which are used as target for regression estimate. The features are scaled to $[0,1]$ and a value of $\epsilon=0.1$ and $\delta=0.9$ is used for regression. 

\begin{table}[t]
\caption[Regression]{Regression results($\dagger$ denote RBF kernel)} 
\centering
\label{tab:regression}       

\begin{tabular}{ccccccc}
\hline\noalign{}
 & \multicolumn{2}{c}{RandSVM} & \multicolumn{2}{c}{LIBSVM} &\multicolumn{2}{c}{$SVM^{Light}$} \\
\hline
   & time & MSE($\rho)$  & time & MSE($\rho)$ & time & MSE($\rho)$ \\
\hline
1.(a)$\dagger$ & 42   & 2.3259(0.9502) & 5.61 & 1.8249(0.922) & 21.10 & 2.2897(0.9509)\\
1.(b) & 1489 & 1.3813(0.9727) & 913.6 & 2.9916(0.9253) & 4114.66 & 1.2173(0.9753) \\
2.(a) & 201 & 0.0319(0.4625) & $\begin{array}{l}
                                 2650.3 \\
                                 1645.8^\dagger
                                \end{array}
$ & $\begin{array}{l}
                                 0.0320(0.4600) \\
                                 0.02621(0.3502)^\dagger
                                \end{array}$ & 336.64 & 0.0320(0.4607) \\
2.(b) & 327 & $68.24\%$ & 4459.8 & $68.49\%$ & 570.07 & $68.32\%$ \\
3. &  713 & 0.0320 (0.7894)&  5671.2$^\dagger$ & 0.0315(0.769755)$^\dagger$ & 460.36 & 0.0317(0.7896)\\
\noalign{\smallskip}\hline
\end{tabular}
\end{table}

\section{Conclusions}\label{sec:conclusions}
A large number of learning problems can be viewed as instances of abstract optimization problem (AOP), which has an associated combinatorial dimension $\Delta$. An AOP can be solved efficiently, with a high degree of accuracy, by selecting subsets of the size of order of the combinatorial dimension of the problem. However, computing the combinatorial dimension of an AOP is not a trivial task. In this paper, we have used ideas from random projections to  obtain estimates to the combinatorial dimension for SVM formulations of classification and regression tasks with extremely promising results.


\end{document}